\documentclass[journal]{IEEEtran}

\ifCLASSINFOpdf
  \usepackage{graphicx}
\else
   \usepackage[dvips]{graphicx}
\fi
\usepackage{url}

\usepackage{amsmath}
\usepackage{amssymb}
\usepackage{mathptmx}
\usepackage{amsfonts}
\usepackage{amsthm}
\newtheorem{theorem}{Theorem}
\newtheorem{corollary}{Corollary}
\newtheorem{proposition}{Proposition}
\usepackage{mathtools}
\usepackage[utf8]{inputenc}
\newcommand{\E}{\mathbb{E}}
\newcommand{\kmax}{\kappa_{\max}}

\begin{document}

\title{A 1/R Law for Kurtosis Contrast in Balanced Mixtures}

\author{Yuda Bi, Wenjun Xiao, Linhao Bai, Vince Calhoun, \IEEEmembership{Fellow, IEEE}
\thanks{Y. Bi is with the Tri-Institutional Center for Translational Research in Neuroimaging and Data Science (TReNDS), Georgia State University, Georgia Institute of Technology, and Emory University, Atlanta, GA 30303 USA (e-mail: ybi@gsu.edu).}
\thanks{W. Xiao is with the Department of Computer Science, The George Washington University, Washington, DC 20052 USA.}
\thanks{L. Bai is with the School of Electrical and Computer Engineering, Georgia Institute of Technology, Atlanta, GA 30332 USA.}
\thanks{V. D. Calhoun is with the Tri-Institutional Center for Translational Research in Neuroimaging and Data Science (TReNDS), Georgia State University, Georgia Institute of Technology, and Emory University, Atlanta, GA 30303 USA, and also with the School of Electrical and Computer Engineering, Georgia Institute of Technology, Atlanta, GA 30332 USA (e-mail: vcalhoun@gsu.edu).}}

\markboth{IEEE Signal Processing Letters}
{Bi \MakeLowercase{\textit{et al.}}: Kurtosis Contrast Decay and Purification}
\maketitle

\begin{abstract}
Kurtosis-based Independent Component Analysis (ICA) weakens in wide, balanced mixtures. We prove a sharp redundancy law: for a standardized projection with effective width $R_{\mathrm{eff}}$ (participation ratio), the population excess kurtosis obeys $|\kappa(y)|=O(\kappa_{\max}/R_{\mathrm{eff}})$, yielding the order-tight $O(c_b\kappa_{\max}/R)$ under balance (typically $c_b=O(\log R)$). As an impossibility screen, under standard finite-moment conditions for sample kurtosis estimation, surpassing the $O(1/\sqrt{T})$ estimation scale requires $R\lesssim \kappa_{\max}\sqrt{T}$. We also show that \emph{purification}---selecting $m\!\ll\!R$ sign-consistent sources---restores $R$-independent contrast $\Omega(1/m)$, with a simple data-driven heuristic. Synthetic experiments validate the predicted decay, the $\sqrt{T}$ crossover, and contrast recovery.
\end{abstract}


\IEEEpeerreviewmaketitle

\section{Introduction}

Independent Component Analysis (ICA) recovers statistically independent latent sources from linear mixtures and is identifiable whenever at most one source is Gaussian~\cite{Comon1994}. This underpins Infomax~\cite{bell1995}, FastICA~\cite{Hyvarinen1999}, JADE~\cite{Cardoso1993JADE}, and related methods~\cite{HyvarinenOja2000,Cardoso1998}, with broad use in neuroimaging~\cite{Calhoun2001,Calhoun2012Review} and telecommunications. Excess kurtosis---the standardized fourth cumulant---is a central contrast function~\cite{Cardoso1989}, and kurtosis-type nonlinearities remain standard in FastICA. While kurtosis-based estimators are well studied~\cite{ChenBickel2006,MiettinenNordhausen2015,Koldovsky2006}, existing analyses treat contrast strength as given rather than quantifying how much contrast \emph{survives} as mixtures widen.

As effective width \(R\) grows, the CLT drives standardized projections toward Gaussianity~\cite{Comon1994}, flattening kurtosis contrast---a familiar effect at high model orders that lacks a population-level scaling law. Herrmann and Theis~\cite{HerrmannTheis2007} studied how finite-sample kurtosis estimation degrades as sources approach Gaussianity, but targeted estimation rather than population contrast. Here we prove a \emph{population-level impossibility}: under balanced projections, true excess kurtosis decays as \(1/R_{\mathrm{eff}}\), so increasing \(T\) alone cannot avert contrast collapse. Large-dimensional ICA studies~\cite{LiJordan2021,AuddyYuan2025,RicciFeature2025} address recovery error and convergence, complementing the population decay characterized here.

This regime arises naturally in neuroimaging: in group ICA~\cite{Calhoun2001,Calhoun2012Review,abou2011group}, higher model orders~\cite{LiAdali2007} increase active sources; group-level projections activate blocks of \(R\) subject-level components, and balance holds when no single component dominates. Proposition~\ref{prop:block_balance} shows this is generic for well-conditioned blocks, with \(c_b\) growing at most logarithmically. Our law thus explains a common pattern: growing model order raises \(R_{\mathrm{eff}}\), shrinking population kurtosis as \(1/R_{\mathrm{eff}}\) and yielding noisy, unreproducible components. Similar effects appear in large-scale multimodal fusion~\cite{silva2020TIP,Silva2024HBM}.

This decay is distinct from the usual \(O(1/\!\sqrt{T})\) estimation error: even with unlimited data, a broad balanced mixture has vanishing kurtosis. The bound characterizes contrast available to generic projections; ICA must locate the rare unbalanced directions that isolate sources, and this search becomes harder as the landscape flattens. We provide an explicit scaling law, a computable model-order ceiling, and a principled mechanism---purification---that restores non-vanishing contrast. This letter contributes:
\begin{itemize}
\item[(i)] A \emph{population-level impossibility law} (Theorem~\ref{thm:redundancy}): for balanced \(R\)-term mixtures, excess kurtosis scales as \(O(1/R)\), and this is order-tight.
\item[(ii)] A \emph{computable model-order diagnostic} (Corollary~\ref{cor:model_order}): a necessary (not sufficient) viability condition is \(R\lesssim\kmax\sqrt{T}\), linking mixture width to sample size.
\item[(iii)] A \emph{purification lower bound} (Theorem~\ref{thm:purification}): selecting \(m\ll R\) sources with a common kurtosis sign---equivalently, a targeted reduction of \(R_{\mathrm{eff}}\)---restores an \(R\)-independent contrast of order \(\Omega(1/m)\).
\end{itemize}

\section{Model and Main Results}

\subsection{Setup and Notation}

We observe \(x_t\in\mathbb{R}^p\) generated from independent sources \(s_t\in\mathbb{R}^k\) via
\begin{equation}\label{eq:csica_model}
x_t = A s_t + \eta_t,\qquad A \in \mathbb{R}^{p\times k},\ \ \E\eta_t=0,
\end{equation}
where \(\eta_t\) is independent noise with finite fourth moments. Sources are standardized: \(\E s_{tj}=0\), \(\E s_{tj}^2=1\) (and for Corollary~\ref{cor:model_order} we assume finite eighth moments to control the variance of the sample kurtosis estimator). For any unit-variance \(v\), define excess kurtosis \(\kappa(v):=\E[v^4]-3\). We present the square case \(p=k\) with invertible \(A\); the rectangular case follows via pseudoinverse.

For a direction \(u\in\mathbb{R}^p\), let \(\mathcal S:=\{j:\,a_j^\top u\neq 0\}\) be the active set and \(R:=|\mathcal S|\) the mixture width. Define normalized projection coefficients
\[
w_j:=\frac{a_j^\top u}{\sqrt{\sum_{\ell\in\mathcal S} (a_\ell^\top u)^2}},\qquad j\in\mathcal S,
\]
so that the standardized noiseless projection can be written as \(y=\sum_{j\in\mathcal S} w_j s_j\) with \(\sum_{j\in\mathcal S} w_j^2=1\). Unless stated otherwise, \(w_j\) are defined for the signal part \(x=As\). Let \(\kmax:=\max_{j\in\mathcal S}|\kappa(s_j)|\). We call the projection \emph{balanced} if \(\max_j |w_j|^2 \le c_b/R\) for some \(c_b>0\). When convenient, we relabel \(\mathcal S=\{1,\dots,R\}\) without loss of generality.

\begin{proposition}[Block balance]\label{prop:block_balance}
Fix a block index set \(\mathcal S\subseteq\{1,\dots,k\}\) with \(|\mathcal S|=R\). If, for a nonzero direction \(u\), the inner products \(c_j:=|a_j^\top u|\) satisfy
\begin{equation}\label{eq:rho_balance}
\max_{j\in\mathcal S} c_j^2 \;\le\; \frac{\rho}{R}\sum_{\ell\in\mathcal S} c_\ell^2
\end{equation}
for some \(\rho\ge 1\), then \(\max_{j\in\mathcal S}|w_j|^2\le \rho/R\).
In particular, if \(A_{\mathcal S}^\top A_{\mathcal S}=I_R\) and \(u\) is uniform on the unit sphere in \(\mathrm{span}(A_{\mathcal S})\), then \(\max_j|w_j|^2\le C(\log R)/R\) with probability at least \(1-2R^{-c}\) for absolute constants \(C,c>0\).
\end{proposition}
\begin{proof}
By definition, \(|w_j|^2=c_j^2/\sum_{\ell\in\mathcal S}c_\ell^2\); taking the maximum and applying \eqref{eq:rho_balance} gives \(\max_j|w_j|^2\le \rho/R\). For the probabilistic statement, under \(A_{\mathcal S}^\top A_{\mathcal S}=I_R\) the vector \(A_{\mathcal S}^\top u\) is uniformly distributed on the unit sphere; the claim follows from Gaussian maxima and chi-square concentration applied to the coordinate maxima of a random unit vector in \(\mathbb{R}^R\) (see, e.g., \cite[Ch. 3]{vershynin2018high}).
\end{proof}

Thus, within a well-conditioned block, a generic direction spreads its energy across \(R\) columns, producing balance with \(c_b\) growing at most logarithmically in \(R\).

\subsection{Redundancy Law}

\begin{theorem}[Sharp Redundancy Bound]\label{thm:redundancy}
Let \(\{s_j\}_{j=1}^R\) be independent sources with unit variance and bounded fourth moments. For any standardized projection
\[
y:=\frac{u^\top x}{\sqrt{\mathrm{Var}(u^\top x)}}
\]
in the noiseless case \(x=As\),
\[
|\kappa(y)| \le \kmax\sum_{j=1}^R |w_j|^4.
\]
Under balanced mixing (\(\max_{j}|w_j|^2 \le c_b/R\)):
\begin{equation}\label{eq:redundancy_bound}
|\kappa(y)| \le \frac{c_b\,\kmax}{R}.
\end{equation}
\emph{Order-tightness.} For equal weights \(|w_j|^2=1/R\) and aligned kurtoses \(\kappa(s_j)=\kappa_0\), one has \(|\kappa(y)|=|\kappa_0|/R\) exactly.
\end{theorem}
\begin{proof}
By independence and unit variance, \(\mathrm{Var}(u^\top As)=\sum_{j=1}^R (a_j^\top u)^2\), so \(y=\sum_{j=1}^R w_js_j\) with \(\sum_{j=1}^R w_j^2=1\). Cross-cumulants vanish, giving \(\kappa(y)=\sum_{j=1}^R w_j^4\kappa(s_j)\) and \(|\kappa(y)| \le \kmax \sum_{j=1}^R |w_j|^4\).
Under balance, \(\sum_j |w_j|^4 \le (\max_j |w_j|^2)\sum_j |w_j|^2 \le c_b/R\). Sharpness: when \(|w_j|^2=1/R\) and \(\kappa(s_j)=\kappa_0\) for all \(j\), we obtain \(|\kappa(y)|=|\kappa_0|/R\).
\end{proof}

Cancellation between positive and negative source kurtoses may further reduce \(|\kappa(y)|\) below this bound; the inequality concerns worst-case magnitude.

\begin{corollary}[Effective-width law]\label{cor:reff}
Define the \emph{effective mixture width} \(R_{\mathrm{eff}}:=1/\sum_{j\in\mathcal S}w_j^4\) (a participation-ratio measure of source dispersion). Then, for any projection,
\begin{equation}\label{eq:reff_bound}
|\kappa(y)|\le \kmax/R_{\mathrm{eff}}.
\end{equation}
Balanced mixing implies \(R_{\mathrm{eff}}\ge R/c_b\), recovering the \(O(1/R)\) rate of Theorem~\ref{thm:redundancy}.
\end{corollary}
\begin{proof}
Immediate from Theorem~\ref{thm:redundancy}: \(\kmax\sum|w_j|^4=\kmax/R_{\mathrm{eff}}\).
\end{proof}
This formulation is fully general---it applies to any weight distribution without the balance assumption. For unbalanced weights, \(R_{\mathrm{eff}}\) can be much smaller than \(R\), explaining the slower decay observed in Fig.~\ref{fig:fig1}(b).

\emph{Remark (noise).}
If \(x=As+\eta\) with independent Gaussian noise, fourth-cumulant additivity gives \(\kappa(y)=\mathrm{SNR}^4\,\kappa(y_s/\sigma_s)/(\mathrm{SNR}^2+1)^2\) where \(\mathrm{SNR}^2:=\sigma_s^2/\sigma_n^2\) is defined at the projected variance level for a fixed scalar projection; the signal kurtosis inherits the \(1/R\) decay under balance, attenuated further by the SNR factor.

\begin{figure*}[t]
  \centering
  \includegraphics[width=0.9\textwidth]{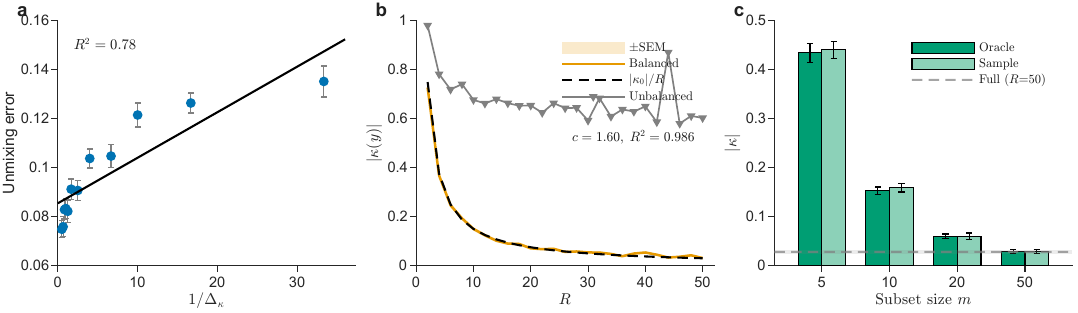}
  \caption{Empirical validation (mean $\pm$ SEM).
  \textbf{(a)} FastICA error $\mathrm{err}(W,A)$ increases with $1/\Delta_\kappa$ ($R^2{=}0.78$).
  \textbf{(b)} Balanced Student-$t$ mixtures (df$=8$, $\kappa_0{=}1.5$) show $|\hat\kappa(y)|\propto 1/R$ for $R=2,\ldots,50$ ($R^2{=}0.986$); unbalanced power-law weights decay more slowly. Inset at $R{=}50$: $\mathrm{std}(\hat\kappa)\sim \sigma_0/\sqrt{T}$ ($\sigma_0\approx5.3$); dotted line: $|\kappa_0|/R$.
  \textbf{(c)} At $R{=}50$ (df$\in[6,30]$), purification ($m{=}5$) increases contrast from $\approx0.03$ to $\approx0.43$ (oracle) / $\approx0.44$ (sample-based), $\sim14\times$.}
  \label{fig:fig1}
\end{figure*}

\emph{Remark (consequence for FastICA).}
Under balance, Theorem~\ref{thm:redundancy} predicts attenuation of kurtosis contrast for standardized candidate projections, so the empirical separation
\(\Delta_\kappa:=\min_{r\neq\ell}|\kappa(y_r)-\kappa(y_\ell)|\) between candidate components can become small. Since FastICA is known to be less well conditioned when kurtosis gaps are small~\cite{Hyvarinen1999,HyvarinenOja2000}, this provides a plausible degradation mechanism (supported by Fig.~\ref{fig:fig1}a), rather than a standalone conditioning theorem. This does not imply that latent-source parameters \(\kappa(s_j)\) change; the shrinkage is mixing-induced attenuation of observable contrast, quantified by Theorem~\ref{thm:redundancy}.

We use \(\sigma_0/\sqrt{T}\) as a conservative detectability scale (one-standard-deviation criterion); stronger guarantees would require algorithm-specific margins.

\begin{corollary}[Necessary model-order screening condition]\label{cor:model_order}
Under the conditions of Theorem~\ref{thm:redundancy}, assume the sample excess-kurtosis estimator \(\hat\kappa\) from \(T\) i.i.d.\ observations satisfies \(\mathrm{std}(\hat\kappa)\le \sigma_0/\sqrt{T}\) for a moment-dependent constant \(\sigma_0>0\) (e.g., under finite eighth moments; see~\cite{MiettinenNordhausen2015}). A necessary condition for the kurtosis contrast of a balanced projection to exceed the estimation noise floor, \(|\kappa(y)|>\sigma_0/\sqrt{T}\), is that its active width \(R\) obeys
\begin{equation}\label{eq:model_order}
R \;<\; \frac{c_b\,\kmax\,\sqrt{T}}{\sigma_0}.
\end{equation}
Equivalently, if an algorithm requires a minimum population contrast \(\kappa^*>0\), then \(R \le c_b\,\kmax/\kappa^*\).
\end{corollary}

\begin{proof}
Theorem~\ref{thm:redundancy} gives \(|\kappa(y)|\le c_b\kmax/R\). Requiring \(c_b\kmax/R>\sigma_0/\sqrt{T}\) yields~\eqref{eq:model_order}. Thus~\eqref{eq:model_order} is an impossibility-screening necessary condition (from an upper bound), not a sufficiency guarantee.
\end{proof}

The bound \(R<c_b\kmax\sqrt{T}/\sigma_0\) is a finite-sample ceiling, while \(R\le c_b\kmax/\kappa^*\) is a population ceiling; the operative constraint is the tighter of the two. The viable range expands only as \(\sqrt{T}\) (doubling tolerable \(R\) requires quadrupling \(T\)). Here \(R\) is projection-dependent and typically smaller than \(k\); conservative screening sets \(R\approx k\), while tighter screening uses the mean activation width. A practical proxy for \(R_{\mathrm{eff}}\) is the participation ratio \(1/\!\sum_j w_j^4\) of a component's loading vector.

\emph{Remark (computable screening).}
Given data \(X\in\mathbb{R}^{p\times T}\) and candidate model order \(k\): (1) whiten and project onto the leading \(k\) principal components; (2) compute sample kurtosis \(\hat\kappa_i\) for each PC direction; (3) set \(\hat\kappa_{\max}=\max_i|\hat\kappa_i|\) and estimate \(\hat\sigma_0\) via bootstrap; (4) choose a conservative log-factor \(c_b\) consistent with Proposition~\ref{prop:block_balance} (e.g., \(c_b=4\log k\)); and (5) compute \(k_{\max}^{(\mathrm{screen})}=c_b\hat\kappa_{\max}\sqrt{T}/\hat\sigma_0\). If \(k\gg k_{\max}^{(\mathrm{screen})}\), kurtosis contrast collapse is expected; reduce model order or apply purification. PC directions are a \emph{coarse proxy} for ICA-active width and are intended to detect hopeless regimes, not to predict contrast precisely.

\subsection{Purification via Sign-Consistent Subset Selection}

Theorem~\ref{thm:redundancy} implies contrast collapses when effective width is large; purification counteracts this by restricting to a sign-consistent subset of \(m\) sources and renormalizing, yielding effective width at most \(m\) and contrast lifted from \(O(1/R_{\mathrm{eff}})\) to \(\Omega(1/m)\), independent of \(R\).

By the pigeonhole principle, among \(R\) sources with non-zero kurtosis, at least \(R/2\) share the same kurtosis sign, so for any target size \(m\ll R\) a sign-consistent subset of size \(m\) can always be selected.

\begin{theorem}[Purification Lower Bound]\label{thm:purification}
Suppose \(\{s_j\}_{j=1}^R\) are independent with non-zero kurtosis. Let \(y=\sum_{j=1}^R w_j s_j\) with \(\sum_{j=1}^R w_j^2=1\).
Given a sign-consistent subset \(\mathcal{M}\subseteq \{1,\dots,R\}\) with \(|\mathcal{M}|=m\) and \(\sum_{j\in\mathcal M}w_j^2>0\), define the purified mixture by restricting and re-normalizing the original weights:
\[
y_{\mathrm{pur}}=\sum_{j\in \mathcal{M}} \tilde w_j s_j,\qquad
\tilde w_j:=\frac{w_j}{\sqrt{\sum_{\ell\in\mathcal M}w_\ell^2}},\qquad
\sum_{j\in \mathcal{M}} \tilde w_j^2=1,
\]
which satisfies the \(R\)-independent lower bound
\begin{equation}\label{eq:purification_bound}
|\kappa(y_{\mathrm{pur}})| \ge \frac{\kappa_{\min,\mathcal{M}}}{m},
\end{equation}
where \(\kappa_{\min,\mathcal{M}}:=\min_{j\in\mathcal{M}}|\kappa(s_j)|\).
\end{theorem}

\begin{proof}
By cumulant additivity, \(\kappa(y_{\mathrm{pur}})=\sum_{j\in\mathcal M}\tilde w_j^4 \kappa(s_j)\).
Under sign-consistency, \(|\kappa(y_{\mathrm{pur}})| \ge \kappa_{\min,\mathcal{M}}\sum_{j\in\mathcal M}\tilde w_j^4\).
By Cauchy--Schwarz, \(\sum_{j\in\mathcal M}\tilde w_j^4 \ge (\sum_{j\in\mathcal M}\tilde w_j^2)^2/m = 1/m\).
\end{proof}

Bound~\eqref{eq:purification_bound} is unconditional: it does not require \(|\kappa(y)|\le C/R\); Theorem~\ref{thm:redundancy} only motivates the regime where purification is most needed. Via Corollary~\ref{cor:reff}, purification yields \(R_{\mathrm{eff}}^{(\mathrm{pur})}\le m\), so \eqref{eq:purification_bound} reads \(|\kappa(y_{\mathrm{pur}})|\ge \kappa_{\min,\mathcal{M}}/R_{\mathrm{eff}}^{(\mathrm{pur})}\). Under additive Gaussian noise, the purified contrast inherits the same SNR attenuation factor as in the noise remark above, since Gaussian noise contributes zero fourth cumulant.

Sign-consistency is essential: if selected sources have mixed kurtosis signs, cancellations can drive \(\kappa(y_{\mathrm{pur}})\) toward zero. In many homogeneous families (e.g., sparse/speech vs.\ bounded/uniform-like sources), a dominant sign is plausible.

A simple data-driven variant requires no oracle. Given candidate components \(\hat s_j\) from a preliminary separation (e.g., PCA + FastICA at moderate model order; even if components remain partially mixed due to contrast degradation, the additive nature of higher-order cumulants robustly ensures that dominant sources dictate the aggregate kurtosis signs), purification selects a subset:
\begin{enumerate}
\item Compute sample kurtoses \(\hat\kappa_j\) of \(\hat s_1,\dots,\hat s_R\).
\item Set \(\mathrm{sign}^*:=\mathrm{sign}\!\bigl(\sum_j\hat\kappa_j\bigr)\).
\item Among indices with \(\mathrm{sign}(\hat\kappa_j)=\mathrm{sign}^*\), select the top-\(m\) by \(|\hat\kappa_j|\).
\item If \(\bigl|\sum_j\hat\kappa_j\bigr|<\tau\) (a small positive empirical threshold, e.g., \(\tau=0.1\)), evaluate both sign choices and keep the subset with larger \(|\hat\kappa(y_{\mathrm{pur}})|\).
\item Project the full data matrix onto the subspace spanned by the estimated mixing vectors of the selected candidates (e.g., using the corresponding columns of the pseudo-inverse of the preliminary unmixing matrix), and re-run ICA on this reduced-dimensional data.
\end{enumerate}
We evaluate this heuristic alongside the oracle in Fig.~\ref{fig:fig1}c, and find comparable contrast restoration.

If the selected subset yields near-zero contrast, this indicates sign cancellation or weak non-Gaussianity; a practical check is to assess the stability of \(\kappa(y_{\mathrm{pur}})\) under bootstrap resampling and then enlarge or re-select \(\mathcal M\) if needed.

\begin{figure}[t]
  \centering
  \includegraphics[width=0.5\textwidth]{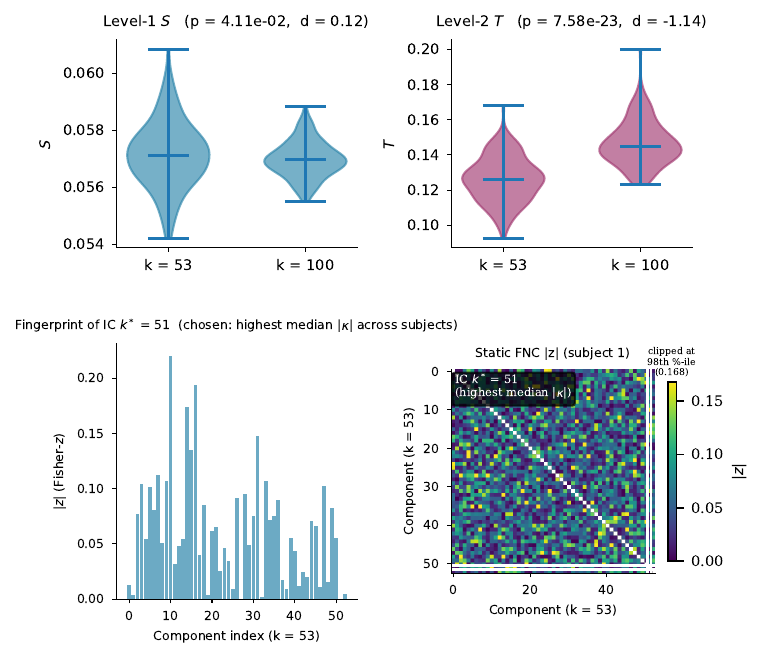}
  \caption{Paired kurtosis-gap comparison in COBRE group ICA ($n=155$) for model orders $k=53$ and $k=100$, with FNC insets.
  Top: two-level FNC summaries (edge-level $S$ and fingerprint-level $F$).
  Bottom: example connectivity fingerprint ($|z|$) and one representative subject's static FNC matrix for component $k^{*}=51$ in the $k=53$ solution (with $|z|$ clipped at the 98th percentile for visualization).
  FNC insets are qualitative context only and are not used in the redundancy-law inference.}
  \label{fig:fig2}
\end{figure}

\section{Experiments}

Sources are independent Student's \(t\)-variables (\(\mathrm{df}>4\)), standardized to zero mean and unit variance, with population excess kurtosis \(\kappa_{\mathrm{t}}(\mathrm{df})=6/(\mathrm{df}-4)\). Results report mean \(\pm\) SEM over independent trials. Sample excess kurtosis is
\(\hat\kappa(v)=\frac{1}{T}\sum_{t}((v_t-\bar v)/\hat\sigma_v)^4-3\).
Unmixing error for symmetric FastICA (\(g(y)=y^3\)) is
\(\mathrm{err}(W,A)=\min_{P,\Lambda}\|WA-P\Lambda\|_F/\sqrt{k}\), where \(P\) is a permutation matrix and \(\Lambda\) is a non-singular diagonal matrix.
(Note: Strictly speaking, guaranteeing the asymptotic \(1/\sqrt{T}\) rate in Corollary~\ref{cor:model_order} requires \(\mathrm{df}>8\) for finite eighth moments. However, for practical finite sample sizes, we empirically observe a stable \(1/\sqrt{T}\) noise-floor crossing even at \(\mathrm{df}=8\).)

\textbf{Conditioning (Fig.~\ref{fig:fig1}a).}
With a well-conditioned \(5\times5\) mixing matrix (\(T=5000\), 30 trials), we vary the kurtosis gap \(\Delta_\kappa\) via source degrees-of-freedom. As \(1/\Delta_\kappa\) increases from \(0.5\) to \(33\), mean unmixing error rises from \(\approx0.07\) to \(\approx0.14\) (\(R^2=0.78\)), consistent with a \(1/\Delta_\kappa\)-type amplification trend.

\textbf{Redundancy (Fig.~\ref{fig:fig1}b).}
Balanced mixtures \(y=\frac{1}{\sqrt{R}}\sum_{j=1}^R s_j\) with i.i.d.\ \(\mathrm{df}=8\) sources (\(\kappa_{\mathrm{t}}=1.5\)), \(T=20{,}000\), and \(R=2,\ldots,50\) exhibit
\(|\hat\kappa(y)|\approx c_{\mathrm{fit}}/R\) with \(c_{\mathrm{fit}}\approx1.60\) (\(R^2=0.986\)), confirming the \(O(1/R)\) law. An unbalanced power-law mixture decays more slowly, consistent with \(R_{\mathrm{eff}}\ll R\) (Corollary~\ref{cor:reff}).
The inset shows \(\mathrm{std}(\hat\kappa)\sim\sigma_0/\sqrt{T}\) (\(\sigma_0\approx5.3\)), crossing the population contrast near \(T\approx3\times10^4\), matching~\eqref{eq:model_order}.

\textbf{Purification (Fig.~\ref{fig:fig1}c).}
At \(R=50\) with heterogeneous sources (\(\mathrm{df}\in[6,30]\), \(T=15{,}000\)), the full balanced mixture yields weak contrast (\(|\hat\kappa|\approx3\times10^{-2}\)). Oracle purification selecting the top-\(m\) positive-kurtosis sources restores contrast to \(\approx0.43\) at \(m=5\) (\(\sim14\times\) gain), while a sample-based sign estimator (\(\tau=0.1\)) achieves comparable results. The monotone decay with \(m\) follows the \(1/m\) scaling of Theorem~\ref{thm:purification}.

\textbf{COBRE real-data sanity check (supportive, time-course level).}
We test an observable implication of the redundancy law using resting-state fMRI group-ICA outputs from the COBRE cohort ($n=155$). We run two group decompositions with model orders $k\in\{53,100\}$. For each subject, component time courses are z-scored across time, and we compute the sample excess kurtosis $\hat\kappa$ for each component. We summarize within-subject non-Gaussian contrast by the kurtosis-gap statistic $G:=\mathrm{mean}(\text{top-5 }|\hat\kappa|)-\mathrm{median}(|\hat\kappa|)$, and compare $G$ across model orders using a paired Wilcoxon signed-rank test (one-sided, testing $G_{53}>G_{100}$). Because time courses are standardized and inference is rank-based, this comparison is insensitive to scale differences between decompositions. Increasing model order from $k=53$ to $k=100$ produces a consistent reduction in $G$ across subjects (Fig.~\ref{fig:fig2}; most paired points lie below the identity line), with a highly significant paired Wilcoxon result ($p=1.74\times10^{-27}$). This provides a within-cohort sanity check of the predicted shrinkage of contrast with increasing model order, while we note that changing $k$ also alters the signal subspace and effective degrees of freedom, so the comparison is supportive rather than a controlled causal test. In the FNC insets, edge-level summary $S$ and fingerprint-level summary $F$ also differ between model orders (Level-1 $S$: $p=4.11\times10^{-2}$, $d=0.12$; Level-2 $F$: $p=7.58\times10^{-23}$, $d=-1.14$). The edge-level effect ($d=0.12$) is small and reported as descriptive context; the primary real-data evidence is the reduction in $G$.

\section{Conclusion}

We established three results for kurtosis-based ICA as mixture width increases: (1) balanced mixing induces \(1/R\) decay of kurtosis contrast; (2) projection-level active width must scale at most as \(\sqrt{T}\) for kurtosis-contrast viability (a necessary screening condition); and (3) sign-consistent subset selection restores non-vanishing contrast. These results help explain high-model-order instability in group neuroimaging: for well-conditioned blocks, balance is typical (Proposition~\ref{prop:block_balance}), so contrast collapse is structural rather than purely algorithmic. The theory is population-level and screening-oriented: it characterizes contrast availability and finite-sample viability, rather than providing a sufficiency guarantee for any specific ICA algorithm. In practice, it offers a simple rationale for model-order screening and for targeted contrast restoration via purification before re-estimation. A simple purification heuristic is near-oracle in our experiments. Our analysis is specific to kurtosis-based contrast and linear instantaneous ICA; other criteria (e.g., negentropy), structured/sparse mixing, and nonlinear or convolutive settings~\cite{hyvarinen2023nonlinear} require separate study. Future work includes adaptive purification within ICA and extensions to multi-domain settings such as IVA-L~\cite{anderson2014independent,kim2006independent}.

\bibliographystyle{IEEEtran}
\bibliography{ref}

\end{document}